\documentclass[twoside]{article}

\usepackage{aistats2021}
\usepackage[round]{natbib}

\usepackage{url}
\usepackage[small]{caption}
\usepackage{booktabs}
\usepackage{amsfonts}
\usepackage{nicefrac}
\usepackage{microtype}
\usepackage{algorithm}
\usepackage{algorithmicx}
\usepackage[noend]{algpseudocode}
\usepackage[belowskip=-5pt,aboveskip=5pt]{caption}
\usepackage{wrapfig}
\usepackage{xspace}
\usepackage{dirtytalk}
\usepackage{graphicx}
\usepackage{amsmath}
\usepackage{amssymb}
\usepackage{amsthm}
\usepackage{bbm}
\usepackage{bm}
\usepackage{dsfont}
\usepackage{subcaption}
\usepackage{times}
\usepackage{mathtools}

\usepackage[normalem]{ulem}

\usepackage[usenames,dvipsnames]{xcolor}
\usepackage[bookmarks=false]{hyperref}
\hypersetup{
  pdftex,
  pdffitwindow=true,
  pdfstartview={FitH},
  pdfnewwindow=true,
  colorlinks,
  linktocpage=true,
  linkcolor=Green,
  urlcolor=Green,
  citecolor=Green
}
\usepackage[capitalize,noabbrev]{cleveref}

\usepackage{color}
\usepackage{tikz}
\usetikzlibrary{bayesnet}
\usetikzlibrary{shapes,decorations,arrows,calc,arrows.meta,fit,positioning}
\tikzset{
    -Latex,auto,node distance =1 cm and 1 cm,semithick,
    state/.style ={ellipse, draw, minimum width = 0.7 cm},
    point/.style = {circle, draw, inner sep=0.04cm,fill,node contents={}},
    bidirected/.style={Latex-Latex,dashed},
    el/.style = {inner sep=2pt, align=left, sloped}
}
\definecolor{light-gray}{gray}{0.95}

\usepackage[backgroundcolor=white]{todonotes}
\newtheorem{proposition}{Proposition}
\newtheorem{lemma}{Lemma}

\newtheorem{theorem}{Theorem}

\algnewcommand{\LineComment}[1]{\State\(\vartriangleright\) #1}

\def\Sset{\mathcal{S}}
\def\Aset{\mathcal{A}}
\def\Xset{\mathcal{X}}

\renewcommand{\hat}{\widehat}
\mathchardef\mhyphen="2D

\newcommand{\indicator}[1]{\mathbbm{1}\left\{#1\right\}}
\newcommand{\E}[2]{\mathbb{E}_{#1} \left[#2\right]}
\newcommand{\prob}[1]{\mathbb{P}\left(#1\right)}

\newcommand{\abs}[1]{\left|#1\right|}

\newcommand{\mexp}{\ensuremath{\tt Exp3.S}\xspace}
\newcommand{\mmexp}{\ensuremath{\tt Exp4.S}\xspace}
\newcommand{\lints}{\ensuremath{\tt CD\mhyphen LinTS}\xspace}
\newcommand{\linucb}{\ensuremath{\tt CD\mhyphen LinUCB}\xspace}
\newcommand{\mucb}{\ensuremath{\tt SW\mhyphen mUCB}\xspace}
\newcommand{\umucb}{\ensuremath{\tt SW\mhyphen umUCB}\xspace}
\newcommand{\mts}{\ensuremath{\tt mTS}\xspace}
\newcommand{\umts}{\ensuremath{\tt umTS}\xspace}
\newcommand{\ts}{\ensuremath{\tt CD\mhyphen TS}\xspace}
\newcommand{\ucb}{\ensuremath{\tt CD\mhyphen UCB}\xspace}

\def\Regret{\mathcal{R}}
\def\Bregret{\mathcal{BR}}

\begin{document}
% \setlength{\abovedisplayskip}{10pt}
% \setlength{\belowdisplayskip}{10pt}
% \setlength{\abovedisplayshortskip}{10pt}
% \setlength{\belowdisplayshortskip}{10pt}



\twocolumn[

\aistatstitle{Non-Stationary Latent Bandits}

\aistatsauthor{Joey Hong \And Branislav Kveton \And  Manzil Zaheer \And Yinlam Chow} \vspace{10pt}

\aistatsauthor{Amr Ahmed \And Mohammad Ghavamzadeh \And Craig Boutilier}  \vspace{10pt}

\aistatsaddress{Google Research}]

\begin{abstract}
Users of recommender systems often behave in a non-stationary fashion, due to their evolving preferences and tastes over time. In this work, we propose a practical approach for fast personalization to non-stationary users. The key idea is to frame this problem as a latent bandit, where the prototypical models of user behavior are learned offline and the latent state of the user is inferred online from its interactions with the models. We call this problem a non-stationary latent bandit. We propose Thompson sampling algorithms for regret minimization in non-stationary latent bandits, analyze them, and evaluate them on a real-world dataset. The main strength of our approach is that it can be combined with rich offline-learned models, which can be misspecified, and are subsequently fine-tuned online using posterior sampling. In this way, we naturally combine the strengths of offline and online learning.
\end{abstract}

\section{Introduction}
\label{sec:introduction}

When users interact with recommender systems or search engines, their behavior is often guided by a \emph{latent state}, a context that cannot be observed. Examples of latent states are \emph{user preferences}, which persist over longer periods of time, and shorter-term \emph{user intents}. As the users interact, their latent state is slowly revealed by their responses. A good recommender should cater to the user based on the latent state, which first needs to be discovered.

We formalize the problem of recommending to a user under a changing latent state as a \emph{multi-armed bandit} \citep{lai1985,auer2002,bandit_book}. In this setting, the recommender is a learning agent and its actions are the arms of a bandit. After an arm is pulled, the agent observes a response from the user, which is also its reward. The response is a function of the observed context and an unobserved latent state. The goal of the learning agent is to maximize its cumulative reward over $n$ interactions with the user. The challenge is that the latent state of the user is unobserved and changes. This setting is known as \emph{piecewise-stationary bandits} \citep{hartland2007,sw_ucb,wmd}.

Both non-stationary bandits \citep{exp3_s,nonstationary_contextual_colt} and the special case of piecewise-stationary bandits \citep{hartland2007,sw_ucb,wmd} have been studied extensively in prior work. The main departures in this work are two fold. First, we assume that the latent state changes stochastically. Second, we assume that the learning agent knows, at least partially, the reward models of arms conditioned on each latent state. This assumption is realistic in most recommender domains, where a plethora of offline data allow for rich models of user behavior, conditioned on the user type, to be learned offline. Under these assumptions, the problem of learning to act can be solved efficiently by Thompson sampling (TS) \citep{thompson33likelihood,chapelle11empirical,russo_posterior_sampling} over latent states, which we propose, analyze, and extensively evaluate. To the best of our knowledge, this is the first analysis of TS in this highly practical setting. 

Our approach has many benefits over prior works. Unlike adversarial techniques \citep{exp3_s,nonstationary_contextual_colt}, we leverage the stochastic nature of the environment, which results in practical algorithms. Unlike stochastic algorithms, which either passively \citep{d_ucb,sw_ucb} or actively \citep{wmd,cpbayesian,cao_cpdetect} adapt to the environment, our algorithms never forget the past or reset their model. In a sense, our approach is the most natural technique under the assumption of knowing, at least partially, the model of the environment. This assumption is natural in any domain where a plethora of offline data is available and leads to major gains over prior work.

Our paper is organized as follows. In \cref{sec:setting}, we introduce our setting of non-stationary latent bandits. In \cref{sec:algorithms}, we propose two posterior sampling algorithms: one knows the exact model of the environment and the other knows a prior distribution over potential models. In \cref{sec:analysis}, we derive gap-free bounds on the $n$-round regret of both algorithms. The algorithms are evaluated in \cref{sec:experiments}. Finally, we discuss related work in \cref{sec:related work} and conclude in \cref{sec:conclusions}.

% \todob{Some recent references to add (do not have to be discussed in detail):

% A Simple Approach for Non-stationary Linear Bandits (AISTATS 2020): Linear bandits with drifts. UCB with periodic resets (no change-point detector). $O(n^{2 / 3})$ regret.

% Learning to Optimize under Non-Stationarity (AISTATS 2019): Linear bandit with drifts. Sliding window UCB. $O(n^{2 / 3})$ regret with tuning and $O(n^{3 / 4})$ without tuning.

% Weighted Linear Bandits for Non-Stationary Environments (NeurIPS 2019): Linear bandit with drifts. Discounted UCB. $O(n^{2 / 3})$ regret.

% We can cite these results earlier and then use them as examples that $O(n^{2 / 3})$ regret is common in UCB-like designs. Also note that one major difference in our work is that our models can be arbitrary (many contexts and actions).
% }

\section{Setting}
\label{sec:setting}

We adopt the following notation. Random variables are capitalized. Greek letters denote parameters and we explicitly state beforehand when they are random. The set of arms is $\Aset = [K]$, the set of contexts is $\Xset$, and the set of latent states is $\Sset$, with $|\Sset| \ll K$. 

The \emph{latent bandit} \citep{latent_bandits} is an online learning problem, where the learning agent interacts with an environment over $n$ rounds as follows. In round $t \in [n]$, the agent observes context $X_t \in \Xset$, chooses action $A_t \in \Aset$, then observes reward $R_t \in \mathbb{R}$. The random variable $R_t$ depends on the context $X_t$, action $A_t$, and latent state $S_t \in \Sset$. The \emph{history} up to round $t$ is
\begin{align*}
  \mathcal{H}_t
  = (X_1, A_1, R_1, \hdots, X_{t - 1}, A_{t - 1}, R_{t - 1})\,.
\end{align*}
The \emph{policy} of the agent in round $t$ is a mapping from its history $\mathcal{H}_t$ and context $X_t$ to the choice of action $A_t$. In prior work \citep{latent_bandits,latent_contextual_bandits,latent_bandits_revisited}, the latent state is assumed to be constant over all rounds, which we relax in this work.

The reward is sampled from a \emph{conditional reward distribution}, $P(\cdot \mid A, X, S; \theta)$, which is parameterized by reward model parameters $\theta \in \Theta$, where $\Theta$ is the space of feasible reward models of the environment. Let $\mu(a, x, s; \theta) = \E{R \sim P(\cdot \mid a, x, s; \theta)}{R}$ be the \emph{mean reward} of action $a$ in context $x$ and latent state $s$ under model $\theta$. 
We assume that the rewards are $\sigma^2$-sub-Gaussian with variance proxy $\sigma^2$,
\begin{align*}
  \E{R \sim P(\cdot \mid a, x, s; \theta)}{\exp[\lambda (R - \mu(a, x, s; \theta))]}
  \leq \exp\left[\frac{\sigma^2 \lambda^2}{2}\right],
\end{align*}
for all $a$, $x$, $s$, and $\lambda > 0$. 
%\todom{Isn't in standard sub-Gaussianity definition, we need the condition to be true for all $\lambda$? not just positive?}
Note that we do not make strong assumptions about the form of the reward: $\mu(a, x, s; \theta)$ can be any complex function of $\theta$, and contexts can be generated by any arbitrary process.

In the \emph{non-stationary latent bandit}, we additionally consider latent states that evolve over time. The initial latent state is drawn according to the prior distribution as $S_1 \sim P_1(s)$. Then, in round $t$, the underlying latent state $S_t \in \Sset$ evolves according to $S_t \sim P(\cdot \mid S_{t-1}; \phi)$, where $\phi \in \mathbb{R}^{|\Sset| \times |\Sset|}$ is the transition matrix. 
% \todob{A bit ambiguous. The fact that $\phi$ is a transition matrix does not say anything about the algebraic form of $P(\cdot \mid S_{t-1}; \phi)$.} 
The graphical model is shown in \cref{fig:bandit_graph}.
\begin{figure}
\centering
\scalebox{0.8}{\begin{tikzpicture}
  % Nodes
  \node[state] (s) {$S_t$} ; %
  \node[state, fill=light-gray, below= of s]    (r) {$R_t$} ; %
  \node[state, fill=light-gray, below=of r]     (x) {$X_t$} ; %
  \node[state, rectangle, fill=light-gray, right=of r, xshift=-0.2cm, yshift=-0.4cm]    (a) {$A_t$} ; %
      
  \node[left=of s, xshift=0.4cm] (prev_s) {$\ldots$} ; 
  \node[right=of s, xshift=-0.4cm] (next_s) {$\ldots$} ; %
  
  \node[state, left=of prev_s, xshift=0.4cm] (s_1) {$S_1$} ; 
  \node[state, right=of next_s, xshift=-0.4cm] (s_n) {$S_n$} ; %
  \node[below=of s_1] (r_1) {$\ldots$} ; 
  \node[below=of s_n] (r_n) {$\ldots$} ; %

  \path (x) edge (a);
  \path (x) edge (r);
  \path (a) edge (r);
  \path (s) edge   (r);

  \path (prev_s) edge (s);
  \path (s) edge (next_s);
  \path (s_1) edge (prev_s);
  \path (next_s)  edge (s_n);
  \path (s_1) edge (r_1);
  \path (s_n)  edge (r_n);
\end{tikzpicture}}
\caption{Graphical model for non-stationary latent bandits.}
\label{fig:bandit_graph}
\vspace{-0.1in}\end{figure}
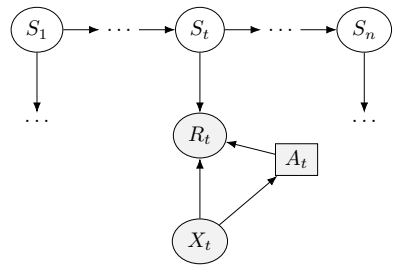
This is useful for applications in which user preferences, tasks, or intents change. For example, the latent states $\Sset$ could be different behavior modes that the user switches between them over time. 

Let $\theta_*, \phi_*$ be the true model parameters, so that the reward in round $t$ is sampled as $R_t \sim P(\cdot \mid A_t, X_t, S_t; \theta_*)$, and the next latent state is sampled as $S_{t+1} \sim P(\cdot \mid S_t; \phi_*)$. Note that the next round's context and latent state are unaffected by the action chosen in the previous round. This is a specific case of POMDPs, where the actions taken by the agent do not affect the dynamics of the environment.

% \todob{The main challenge with introducing regret here is that we consider three regrets later: 1) $S_t$ is random, 2) also $\theta_\ast$ is random, and 3) also $\phi_\ast$ is random. I suggest the following. Introduce regret here under the assumption that everything fixed. Then say that we will integrate later over some quantities.}

Performance of bandit algorithms is typically measured by regret. For a variable $X$, let $X_{i:j}$ denote its concatenation from rounds $i$ to $j$, inclusive. For a fixed latent state sequence $s_{1:n} \in \Sset^n$ and model $\theta_* \in \Theta$, let $A_{t, *} = \arg\max_{a \in \Aset} \mu(a, X_t, s_t, \theta_*)$ be the optimal arm. Then the \emph{expected $n$-round regret} is defined as
\begin{align}
    &\Regret(n; \theta_*, s_{1:n}) \label{eqn:regret} \\
    &\,= \E{}{\sum_{t=1}^n \mu(A_{t, *}, X_t, s_t; \theta_*) - \mu(A_t, X_t, s_t; \theta_*)}\,. \nonumber
\end{align}

In this work, we consider the Bayes regret, which includes an expectation over latent state and/or model randomness. 
We use two different notions of the Bayes regret. The first one is when the true model $\theta_*, \phi_*$ is fixed, and the expectation is only over randomness in latent states. The \emph{$n$-round Bayes regret with fixed model} is
\begin{align}
\Bregret(n; \theta_*, \phi_*) 
=\E{S_{1:n} \sim \phi_*}{\Regret(n; \theta_*, S_{1:n}) \mid \theta_*, \phi_*} \label{eq:bayes_regret} 
\end{align}
where $A_{t, *} = \arg\max_{a \in \Aset} \mu(a, X_t, S_t; \theta_*)$ is also a function of the random latent state.
% \todomgh{The conditioning over $\theta_*$ and $\phi_*$ seems redundant, because they are both given as the arguments to the Bayesian regret.} \todob{This is ok. BR itself is a conditional expectation, parameterized by what we condition on.}
We also study the case where the true model is sampled from prior $\theta_*, \phi_* \sim P_1$. Then $A_{t, *}$ depends on the \emph{random latent state and model}, and the \emph{$n$-round Bayes regret} is
\begin{align}
  \Bregret(n)
  & = \E{}{\Bregret(n; \theta_*, \phi_*)}\,.
\label{eq:bayes_regret_2}
\end{align}
It is important to note that the Bayes regret is a weaker metric than regret, which is worst case over latent sequences and models. 
However, we are often more concerned in practice with the average performance over a range of latent state sequences and models, that arise with multiple users or multiple sessions with the same user. This quantity is sufficiently captured by Bayes regret. 
% Assuming $S_{1:t}$ is drawn according to $\phi_*$, and $\theta_*, \phi_*$ are drawn from some prior, the \emph{$n$-round Bayes regret} is:
% \begin{align}
%     \Bregret(n) 
%     &= \E{}{\Regret(n; S_{1:n}, \theta_*)} \\
%     &= \E{}{\sum_{t=1}^n \mu(A_{t, *}, X_t, S_t, \theta_*) - \mu(A_t, X_t, S_t; \theta_*)} \nonumber.
% \label{eqn:bayes_regret}
% \end{align}
% \todom{The dependency of $\phi_*$ is not clear. I like Brano's suggestion which makes it clear.}
% Here $A_{t, *} = \arg\max_{a \in \Aset} \mu(a, X_t, S_t, \theta_*)$ additionally depends on random latent state and model.

\section{Model-Based Thompson Sampling}
\label{sec:algorithms}

Recall that in round $t$, we have that $S_{t}$ is the true latent state, and model parameters $\theta_*, \phi_*$ determine the conditional rewards $P(R_t \mid A_t, X_t, S_t; \theta_*)$ and transition probabilities $P(S_t \mid S_{t-1}; \phi_*)$, respectively.

Our proposed algorithm is Thompson sampling (TS) with an offline-learned model. At a high-level, the TS algorithm operates by sampling actions stochastically according to $\prob{A_t = a \mid \mathcal{H}_t, X_t} = \prob{A_{t, *} = a \mid \mathcal{H}_t, X_t}$.
In \cref{sec:known_model}, we consider a simple case where the true model is recovered offline. In most realistic scenarios though, the true model is unknown, and we only know its uncertain estimate. In \cref{sec:unknown_model}, we consider an agnostic case, where only priors over the reward and transition models are known, that is $\theta_* \sim P_1(\theta), \phi_* \sim P_1(\phi)$. We use $P_1$ to denote the prior over all environment parameters, including the initial state $S_1$, and model and transition parameters $\theta_*, \phi_*$. We parameterize the distribution to make it clear which of the environment parameters we refer to.

\subsection{Known Models}
\label{sec:known_model}

First, we propose model-based Thompson sampling (\mts), where the true reward and transition models are known, that is exact $\theta_*, \phi_*$ are recovered offline. In this case, TS reduces to sampling a belief state $B_t \in \Sset$ from its posterior distribution over latent states, and acting according to $B_t$ and model parameters. In particular, $A_t = \arg\max_{a \in \Aset}\mu(a, X_t, B_t; \theta_*)$. The pseudocode of \mts is detailed in \cref{alg:thompson_1}. In \eqref{eq:posterior}, we compute the posterior as a filtering distribution $P_t(s) = \prob{S_t = s \mid \mathcal{H}_t}$. Since the model is known exactly, this can be computed as an incremental update from $P_{t-1}$. Then, after sampling $B_t$ from the posterior, the algorithm simply chooses the best-performing action from the conditional reward model for $B_t$. 

\begin{algorithm}[ht]
\caption{\mts}\label{alg:thompson_1}
\begin{algorithmic}[1]
  \State \textbf{Input:}
  \State \quad Model parameters $\theta_*, \phi_*$
  \State \quad Prior over initial latent state $P_1(s)$
  \Statex
  \For {$t \gets 1, 2, \hdots$}
    \State Sample $B_t \sim P_t$
    \State Select 
    $A_t \leftarrow \arg\max_{a \in \Aset} \mu(a, X_t, B_t; \theta_*)$
    \State Observe $R_t$. Update posterior
    \begin{align}
    &P_{t+1}(s_{t+1}) \propto \label{eq:posterior} \\
    &\, \sum_{s_t \in \Sset} P_{t}(s_t) P(s_{t+1} \mid s_t; \phi_*)) P(R_t \mid A_t, X_t, s_t; \theta_*) \nonumber
    \end{align}
\EndFor
\end{algorithmic}
\end{algorithm}\vspace{-0.1in}

\subsection{Uncertain Models}
\label{sec:unknown_model}

As alluded to earlier, it is unrealistic to assume that the true model parameters $\theta_*, \phi_*$ can be recovered from offline data. Because of this, many methods in prior literature attempt to learn uncertainty over the model, sometimes called epistemic uncertainty \citep{model_uncertainty}, in the form of a prior over model parameters. In practice, learning such prior may be intractable for complex models, but can be approximated, for instance by an ensemble of bootstrapped models \citep{model_uncertainty}.

We propose uncertainty-aware model-based Thompson sampling (\umts), where the reward and transition models are estimated with uncertainty. Formally, we are given priors $P_1(\theta), P_1(\phi)$ such that $\theta_* \sim P_1(\theta), \phi_* \sim P_1(\phi)$.
In \umts, we maintain a joint posterior distribution $P_t(s, \theta) = \prob{S_t = s, \theta_* = \theta \mid \mathcal{H}_t}$, sample a believed latent state and reward model $B_t, \theta$ from this distribution, and act according to $A_t = \arg\max_{a \in \Aset} \mu(a, X_t, B_t; \theta)$. The joint posterior is given in \eqref{eq:joint_posterior} and the algorithm is detailed in \cref{alg:thompson_2}. Because transition parameters $\phi$ are not used for decision making, they get marginalized in the posterior.

\begin{algorithm}[ht]
\caption{\umts}\label{alg:thompson_2}
\begin{algorithmic}[1]
  \State \textbf{Input:}
  \State \quad Prior over model parameters $P_1(\theta), P_1(\phi)$
  \State \quad Prior over initial latent state $P_1(s)$
  \Statex
  \State Initialize $P_1(s, \theta) \propto P_1(s_1)P_1(\theta)$
  \For {$t \gets 1, 2, \hdots$}
    \State Sample $B_t, \theta \sim P_t$
    \State Select 
    $A_t \leftarrow \arg\max_{a \in \Aset} \mu(a, X_t, B_t; \theta)$
    \State Observe $R_t$. Update joint posterior
    \begin{align}
    &P_{t+1}(s_{t+1}, \theta) \propto \label{eq:joint_posterior} \\
    &\hspace{-5pt}\int_{\phi} P_1(\theta, \phi)\smashoperator[lr]{\sum_{s_{1:t} \in \Sset^t}}P(s_{1:t+1} \mid \phi) P(\mathcal{H}_{t+1} \mid s_{1:t}; \theta) d\phi \nonumber
    \end{align}
\EndFor
\end{algorithmic}
\end{algorithm}\vspace{-0.1in}

Note that the joint posterior in \eqref{eq:joint_posterior} requires a summation over past latent state trajectories and is therefore intractable. We propose and analyze \cref{alg:thompson_2} as a computation-inefficient algorithm, but approximate it using sequential Monte Carlo (SMC) in practice \citep{smc}.

\subsection{Approximate Inference for Uncertain Models}
In this section, we propose and approximate SMC algorithm to \umts.
Particularly, we use particle filtering with $N$ particles \citep{smc,sarkka2013bayesian}, where each particle maintains its own latent state trajectory. At round $t$, particle $i$ independently samples believed state and model $B_t^{(i)}, \theta^{(i)}, \phi^{(i)} \sim P_t^{(i)}$, where joint posterior $P_t^{(i)}\left(s, \theta, \phi\right) = \prob{S_t = s, \theta_* = \theta, \phi_* = \phi \mid \mathcal{H}_t, B_{1:t-1}^{(i)}}$ additionally depends on the particle's past latent trajectory. 

For each round $t$, the SMC algorithm maintains a weight $w_t$ over particles, and acts according to the weighted average of the particles' latent state and model parameters. The weights for all particles are updated using the incremental likelihood of the resulting observations in round $t$ as in \eqref{eq:particle_weight_update} and renormalized. 
% Note that the numerator in the update term of \eqref{eq:particle_weight_update} can be expanded as,
% \begin{align*}
%     w_{t+1, i} 
%     \leftarrow w_{t, i} \,
%     \frac{P(B_{t}^{(i)} \mid B_{t-1}^{(i)}; \phi^{(i)})P(R_t \mid A_t, X_t, B_t^{(i)}; \theta^{(i)})}
%     {P(B_t^{(i)} \mid R_t, A_t, X_t, B_{t-1}^{(i)}; \theta^{(i)}, \phi^{(i)})}
% \end{align*}
If the current weights $w_t$ satisfy a resampling criterion, then the filtering algorithm resamples $N$ particles in proportion to their weights with replacement. The algorithm is detailed in \cref{alg:thompson_2_smc}.

\begin{algorithm}[ht]
\caption{\umts (Particle Filtering)}\label{alg:thompson_2_smc}
\begin{algorithmic}[1]
  \State \textbf{Input:}
  \State \quad Prior over model parameters $P_1(\theta), P_1(\phi)$
  \State \quad Prior over initial latent state $P_1(s)$
  \State \quad Number of particles $N$
  \Statex
  \State Sample $B_1^{(i)}, \theta_1^{(i)} \sim P_1, i = 1, \hdots, N$
  \State Set $w_1 \in \mathbb{R}^N$ s.t. $w_{1, i} \leftarrow N^{-1}$, $i = 1, \hdots, N$
  \For {$t \gets 1, 2, \hdots$}
    \State For $a \in \Aset$, set $\mu_t(a) \in \mathbb{R}^N$ s.t.
    $$
    \mu_{t, i}(a) \leftarrow \mu(a, X_t, B_t^{(i)}, \theta^{(i)}), \,i  = 1, \hdots, N
    $$
    \State Select
    $A_t \leftarrow \arg\max_{a \in \Aset} w_t \cdot \mu_t(a)$
    \For {$i \gets 1, \hdots, N$}
        \State Sample 
        $
        B_{t+1}^{(i)}, \theta^{(i)}, \phi^{(i)} \sim P_{t+1}^{(i)}
        $
        as in \eqref{eq:joint_posterior_decomposition}
        \State Set $w_{t+1} \in \mathbb{R}^N$ s.t.
        \begin{align}\hspace{-15pt}
        w_{t+1, i} 
        \leftarrow w_{t, i} \,
        \frac{P(R_t, B_{t}^{(i)} \mid A_t, X_t, B_{t-1}^{(i)}; \theta^{(i)}, \phi^{(i)})}
        {P(B_t^{(i)} \mid R_t, A_t, X_t, B_{t-1}^{(i)}; \theta^{(i)}, \phi^{(i)})}
        \label{eq:particle_weight_update}
        \end{align}
    \EndFor
    \State Compute $ESS \leftarrow \left(\sum_{i=1}^N w_{t+1, i}^2 \right)^{-1}$. Resample particles if $ESS$ is too small.
\EndFor
\end{algorithmic}
\end{algorithm}\vspace{-0.1in}

For a matrix (vector) $M$, we let $M_i$ denote its $i$-th row (element). Using this notation, we can write $\theta = (\theta_s)_{s \in \Sset}$ and $\phi = (\phi_s)_{s \in \Sset}$ as vectors of conditional parameters, one for each latent state. We can show that the sampling step for each particle can be done tractably if the reward model prior $P_1(\theta_s)$ and likelihood $P(r \mid x, a, s; \theta)$ are conjugates distributions in the exponential family, and the transition prior for each latent state $P_1(\phi)$ factors as Dirichlet for each state $s$, i.e. $\phi_{*, s} \sim \mathsf{Dir}((\alpha_{s, s'})_{s' \in \Sset})$. 
The key detail is that now we can obtain samples from the joint posterior using the particles and avoid the intractable sum over all possible past trajectories as in \eqref{eq:joint_posterior} in \cref{alg:thompson_2}. 
% The key detail is that the joint posterior that we sample from is conditioned on the past trajectory of each particle, which avoids the intractable sum over all possible past trajectories as in \eqref{eq:joint_posterior} in \cref{alg:thompson_2}. 
For particle $i$, we decompose the joint posterior as
\begin{align}
    &P^{(i)}_t\left(s_t, \theta, \phi\right) \label{eq:joint_posterior_decomposition} \\
    &\,\propto P\left(\phi \mid B^{(i)}_{1:t-1}\right)
    P\left(s_t \mid B^{(i)}_{t-1}; \phi\right)
    P\left(\theta \mid \mathcal{H}_t, B^{(i)}_{1:t-1}\right)\,. \nonumber
\end{align}
Hence, sampling from the joint posterior can be done by first sampling the transition parameters, then believed latent state, and finally reward parameters for that state.

In the case where prior $\phi_{*, s} \sim \mathsf{Dir}((\alpha_{s, s'})_{s' \in \Sset})$ is Dirichlet with parameters $\alpha_s \in \mathbb{R}^{|\Sset|}$, the posterior is also Dirichlet. For particle $i$, the posterior parameters are simply updated with the observed transitions in its latent state trajectory $B_{1:t-1}^{(i)}$. Formally, the posterior over state transitions from state $s$ would be: $\phi^{(i)}_s \mid B^{(i)}_{1:t-1} \sim$ 
\begin{align*}
\mathsf{Dir}\left(\left(\alpha_{s, s'} + \textstyle\sum_{\ell=1}^{t-1} \indicator{B^{(i)}_{t-1} = s, B^{(i)}_t = s'}\right)_{s' \in \Sset}\right)\,.
\end{align*}
% \begin{equation*}
% \resizebox{\linewidth}{!}{%
%     \phi^{(i)}_s \mid B^{(i)}_{1:t-1} 
%     \sim
%     \mathsf{Dir}\left(\left(\alpha_{s, s'} + \textstyle\sum_{\ell=1}^{t-1} \indicator{B^{(i)}_{t-1} = s, B^{(i)}_t = s'}\right)_{s' \in \Sset}\right)\,.}
% \end{equation*}
The transition matrix $\phi^{(i)}$ can be tractably sampled from this Dirichlet posterior. The next latent state $B_t^{(i)}$ is easily sampled from $\phi^{(i)}$.

Recall that we assumed that the reward model prior and conditional reward distribution belong to the exponential family, which covers commonly studied reward distributions, such as Gaussian and Bernoulli. We assume that the reward likelihood is written
\begin{align*}
    P(r \mid a, x, s; \theta) = \exp\left[f(r, a, x)^\top \kappa(\theta_s) - g(\theta_s) \right],
\end{align*}
where $f(r, a, x)$ are sufficient statistics for the observed data, $\kappa(\theta_s)$ are the natural parameters, and $g(\theta_s)$ is the log-partition function. Then, the prior over $\theta_s$ is the conjugate prior of the likelihood, which has the general form of
\begin{align*}
    P_1(\theta_s) \propto \exp\left[\psi_{s, 1}^\top \kappa(\theta_s) - m_{s, 1} g(\theta_s) \right],
\end{align*}
where $\psi_{s, 1}, m_{s, 1}$ are parameters controlling the prior and $H(\psi_{s, 1}, m_{s, 1})$ is the normalizing factor.

For particle $i$, round $t$, and state $s$, updating the posterior over $\theta_s$ simply involves updating the prior parameters with sufficient statistics from the data. Specifically, we have
$m^{(i)}_{s, t} \leftarrow m_{s, 1} + \sum_{\ell=1}^{t-1} \indicator{B^{(i)}_\ell = s}$ and
\begin{align*}
    \psi^{(i)}_{s, t} \leftarrow \psi_{s, 1} + \sum_{\ell=1}^{t-1} \indicator{B^{(i)}_\ell = s}f(R_\ell, A_\ell, X_\ell),
\end{align*}
form the conditional posterior
\begin{align*}
 P(\theta_s \mid \mathcal{H}_t, B^{(i)}_{1:t-1}) 
 \propto \exp\left[\psi^{(i)\,\top}_{s,t} \kappa(\theta_s) - m^{(i)}_{s, t} g(\theta_s) \right].
\end{align*}
Hence, each term in the joint posterior decomposition in \eqref{eq:joint_posterior_decomposition} has an analytic form, and can be tractably sampled from.
% Similarly, for round $t$, the marginal posterior of $s$ is written,
% \begin{align}
% \begin{split}
%   &P(s \mid \mathcal{H}_t, B^{(i)}_{1:t-1}; \phi) \label{eq:marginal_posterior_exponential} \\ 
% %   &\propto \sum_{s_{t-1} \in \Sset} P(s \mid s_{t-1}; \phi_*) P_{t-1}(s_{t-1})
% %   \int_{\theta_{s_{t-1}}} P_{t-1}(\theta_{s_{t-1}}) P(R_t \mid A_t, X_t, s_{t-1}; \theta) \\
%   &\,\propto \sum_{s_{t-1} \in \Sset} P(s \mid s_{t-1}; \phi) P_{t-1}(s_{t-1}) \\
%   &\qquad\int_{\theta_{s_{t-1}}}
%   \exp\left[\psi_{s_{t-1}, t}^\top \kappa(\theta_{s_{t-1}}) - m_{s_{t-1}, t} g(\theta_{s_{t-1}}) \right] \nonumber
%   d \theta \nonumber \\
%   &\,\propto \sum_{s_{t-1} \in \Sset}P(s \mid s_{t-1}; \phi) P_{t-1}(s_{t-1}) H(\psi_{s_{t-1}, t}, m_{s_{t-1}, t}).
% \end{split}
% \end{align}

% Then, the next believed latent state $B^{(i)}_t$ can be sampled from the marginal posterior as in \eqref{eq:marginal_posterior_exponential} using $\phi^{(i)}$ instead of $\phi_*$.
% For each latent state $s$, the reward model parameters $\theta_s^{(i)}$ can be sampled as in \eqref{eq:conditional_posterior_exponential} using particle $i$'s past trajectory $B_{1:t-1}^{(i)}$.

\section{Analysis}
\label{sec:analysis}

In this section, we derive Bayes regret bounds for \mts and \umts. Recall that $A_{t, *}$ is the optimal action in round $t$. The key idea in our analysis is that the conditional distributions of $A_{t, *}$ and $A_t$, as sampled in \mts, are identical. Formally,  $\E{}{f(A_{t, *}) \mid X_t, \mathcal{H}_t} = \E{}{f(A_t) \mid X_t, \mathcal{H}_t}$ for any function $f$ of history $\mathcal{H}_t$ and context $X_t$. Following \citet{russo_posterior_sampling}, we design $f$ as an upper confidence bound (UCB) in a suitable UCB algorithm. In \cref{sec:sliding_window_mucb}, we first propose that algorithm. Then, in \cref{sec:regret_decomposition}, we state a key regret decomposition and show how to derive Bayes regret bounds for our algorithms using the UCB algorithm. In \cref{sec:regret_bounds}, we present our regret bounds.

\subsection{Model-Based UCB}
\label{sec:sliding_window_mucb}

In this section, we propose \mucb, a model-based sliding-window UCB algorithm that uses an offline-learned model to identify non-stationary latent states. In the domain of non-stationary bandits, \citet{d_ucb} and \citet{sw_ucb} proposed two passive adaptations to the UCB algorithm: discounting past observations or ignoring them using a sliding window. Without loss of generality, we focus on the latter due to being better suited for abrupt changes in latent state (as opposed to gradual ones). The algorithm is similar to that proposed by \citet{latent_bandits} and \citet{latent_bandits_revisited} for stationary environments, but augmented with an additional sliding window. The novelty is that the sliding window allows for sublinear regret when the environment is non-stationary. 

\mucb is detailed in \cref{alg:sw-ucb}. At a high level, it takes model parameters $\theta_*$ as an input. We discuss how to change \mucb when $\theta_*$ is not known in the Appendix. \mucb maintains a set of latent states $C_t$ \emph{consistent} with the rewards observed in the most recent $\tau$ rounds, where $\tau$ is a tunable parameter. In round $t$, it chooses a belief state $B_t$ from $C_t$ and the arm $A_t$ with the maximum expected reward in that state, $(B_t, A_t) = \arg\max_{s \in C_t, a \in A} \mu(a, X_t, s; \hat{\theta})$.

In \mucb, the UCB for action $a$ in round $t$ is
\begin{align}
  \textstyle
  U_t(a)
  = \arg\max_{s \in C_t} \mu(a, X_t, s; \hat{\theta})\,.
  \label{eq:ucb}
\end{align}
The consistent latent states are determined by \say{gap} $G_t(s)$, defined in \eqref{eqn:ucb_gap}. If $G_t(s)$ is high, \mucb marks state $s$ as \emph{inconsistent} and does not consider it in estimating UCB $U_t$.

\subsection{Regret Decomposition}
\label{sec:regret_decomposition}

Note that for any action $a \in \Aset$, the upper confidence bound $U_t(a)$ in \eqref{eq:ucb} is deterministic given $X_t$ and $\mathcal{H}_t$. This observation leads to the following regret decomposition.

\begin{proposition}
The Bayes regret of \mts decomposes
\begin{align}
    &\Bregret(n; \theta_*, \phi_*) \label{eqn:posterior_regret_decomposition} \\
    &= \E{}{\sum_{t=1}^n \mu(A_{t, *}, X_t, S_t; \theta_*) - U_t(A_{t, *}) \mid \theta_*, \phi_*} \nonumber \\ 
    &\qquad+ \E{}{\sum_{t=1}^n U_t(A_t) - \mu(A_t, X_t, S_t; \theta_*) \mid \theta_*, \phi_*} \nonumber \,.
\end{align}
\label{prop:regret_decomposition}
\vspace{-0.2in}\end{proposition}

The proof is due to \cite{russo_posterior_sampling}, and follows from rewriting the Bayes regret in terms of $U_t$ and the observation above. Note that while we use the fixed-model formulation of the Bayes regret in \eqref{eq:bayes_regret}, the proposition still holds for general \eqref{eq:bayes_regret_2}.

Hence, though the UCBs $U_t$ are not used by our TS algorithms, they can be used to \emph{analyze} them due to the decomposition in \eqref{eqn:posterior_regret_decomposition}. Specifically, our derivation of a Bayes regret bound for $\mts$ proceeds according to the outline below.

\textbf{Step 1: $S_t \in C_t$ with high probability.} We show that the true latent state is in our consistent sets with a high probability. This means that the first term in \eqref{eqn:posterior_regret_decomposition} is small.

\textbf{Step 2: Regret bound for \mucb.} This follows from bounding both terms in \eqref{eqn:posterior_regret_decomposition}. The second term is the sum of confidence widths over time, or difference between $U_t$ and the true mean reward. The widths decrease, under appropriate conditions, whenever an arm is pulled.

\textbf{Step 3: Bayes regret bound for \mts.} We exploit the fact that the Bayes regret decomposition for \mts in \eqref{eqn:posterior_regret_decomposition} can be equivalently stated for the regret of \mucb. Hence, any UCB regret bound transfers to a TS Bayes regret bound.

\begin{algorithm}[t]
\caption{\mucb}
\begin{algorithmic}[1]
  \State \textbf{Input:} Model parameters $\theta_*$, window size $\tau$
  \Statex
  \For{$t \gets 1, 2, \hdots$}
    \State Define $N_t(s) \leftarrow \sum_{\ell = \max\{1, t - \tau\}}^{t-1}\indicator{B_\ell = s}$ and 
    \begin{align}\hspace{-10pt}
    \label{eqn:ucb_gap}
    G_{t}(s) \leftarrow \smashoperator[lr]{\sum_{\ell = \max\{1, t - \tau\}}^{t-1}} \indicator{B_\ell = s}\left(\mu(A_\ell, X_\ell, s; \theta_*) - R_\ell\right)
    \end{align}
    \State Set of consistent latent states
    $$C_t \leftarrow \left\{s \in S: G_t(s) \leq \sigma \sqrt{6N_t(s)\log n} \right\}$$
    \State Select $B_t, A_t \leftarrow \arg\max_{s \in C_t, a \in A} \mu(a, X_t, s; \theta_*)$
  \EndFor
\end{algorithmic}
\label{alg:sw-ucb}
\end{algorithm}

For Step 3 to hold, our analysis in Step 2 needs to be worst-case over suboptimal latent states and actions. This is why we cannot use the fact that actions $A_t$ maximize $U_t$ in \eqref{eq:ucb}, and derive gap-free bounds.

\subsection{Regret Bounds}
\label{sec:regret_bounds}

In this section, we state Bayes regret bounds for \mts with known model and \umts with uncertain model. As described in \cref{sec:regret_decomposition}, our bounds follow from that on \mucb and \cref{prop:regret_decomposition}. That bound is stated below in terms of the number of stationary segments in a horizon of $n$ rounds, $L = \sum_{t = 2}^n\indicator{s_t \neq s_{t-1}} + 1$. We defer proofs of all claims to Appendix.

\begin{lemma}
For known model parameters $\theta_*$ with $\hat{\theta} = \theta_*$, and optimal choice of $\tau$, the $n$-round regret of $\mucb$ is
\begin{align*}
    \Regret(n; \theta_*, s_{1:n})
    &= \mathcal{O}\left(n^{2/3}\sqrt{|\Sset|L\log n}\right).
\end{align*}
\label{lem:ucb_regret}
\vspace{-0.2in}\end{lemma}

Prior derivations for sliding-window UCB without context achieved a gap-dependent bound of $\tilde{\mathcal{O}}(K \sqrt{n L}/ \Delta^2)$ (see \citealt[Theorem 7]{sw_ucb}) after tuning $\tau$, where $K$ is the number of arms. A gap-free bound can be obtained by bounding $\tilde{\mathcal{O}}(n^{-1 / 6})$ gaps trivially. This yields a $\tilde{\mathcal{O}}(n^{5 / 6})$ regret bound, which is worse than \cref{lem:ucb_regret}.

In practice, the latent state sequence, and hence the number of stationary segments $L$, is often stochastic. Given $\phi_*$, let $p = 1 - \min_{s \in \Sset} P(s \mid s; \phi_*)$ be the maximum probability of a change occurring. We can bound the expected value of $L$ from above by $1 + pn$. This yields the following Bayes regret bound for \mts.

\begin{theorem}
For known model parameters $\theta_*, \phi_*$, let
$p = 1 - \min_{s \in \Sset} P(s \mid s; \phi_*)$
with $\tilde{L} = 1 + p n$. Then, the $n$-round Bayes regret of $\mts$ is
\begin{align*}
    \Bregret(n; \theta_*, \phi_*)
    = \mathcal{O}\left(n^{2/3}\sqrt{|\Sset|\tilde{L}\log n}\right).
\end{align*}
\label{thm:ts_regret}
\vspace{-0.2in}\end{theorem}

Note that recent non-stationary bandit algorithms with active change-point detection have $\tilde{\mathcal{O}}(\sqrt{nKL})$ regret bounds \citep{wmd,cao_cpdetect}, where $K$ is the number of arms. However, such change-point detectors do not easily generalize to scenarios with context, and require knowledge of $n, L$ to tune their hyperparameters optimally. Our algorithm \mts handles context and does not require any parameter tuning. \mucb is simply a tool to construct $U_t$ and analyze \mts; better algorithms may exist that yield tighter regret bounds for \mts. Also, while the expected number of stationary segments $\tilde{L} = pn$ appears linear in $n$, all prior works essentially assume $p = \mathcal{O}(1/n)$ by treating the number of stationary segments as a constant. Since changes are rare in many realistic applications, it is safe to assume that $\tilde{L} = \mathcal{O}(n^{\beta})$, for some small $\beta > 0$.

Our next result is for \umts when only a prior over the reward and transitions is known. Our statement changes in two ways: (i) we introduce a high-probability error $\varepsilon$ in estimating the reward via a sample from the prior, and (ii) the expected number of changes $\tilde{L}$ depends on the transition prior. Recall that for any latent state $s$, we assume that the transition model $\phi_{*, s}$ is sampled as $\phi_{*, s} \sim \mathsf{Dir}((\alpha_{s, s'})_{s' \in \Sset})$. We define $\bar{\mu}(a, x, s) = \int_\theta \mu(a, x, s; \theta) P_1(\theta) d \theta$ as the mean conditional reward, marginalized with respect to the prior.

\begin{theorem}
\label{thm:ts_regret_prior} Let $(\alpha_{s, s'})_{s, s' \in \Sset \times \Sset}$ be the prior parameters of $P_1(\phi)$, such that $\phi_* \sim P_1(\phi)$ factors over state $s$ as $\phi_{*, s} \sim \mathsf{Dir}((\alpha_{s, s'})_{s' \in \Sset})$. Let
$
  p = 1 - \min_{s \in \Sset} \alpha_{s, s} / \sum_{s' \in \Sset} \alpha_{s, s'}
$
and $\tilde{L} = 1 + p n$. For $\theta_* \sim P_1(\theta)$, choose $\varepsilon, \delta > 0$ such that
\begin{align*}
  \left\{\forall a \in \Aset, x \in \Xset, s \in \Sset: |\bar{\mu}(a, x, s) - \mu(a, x, s; \theta_*)| \leq \varepsilon\right\}
\end{align*}
holds with probability at least $1-\delta$. Then, the $n$-round Bayes regret of $\umts$ is
\begin{align*}
    \Bregret(n)
    &= \mathcal{O}\left(\delta n + \varepsilon n + n^{2/3}\sqrt{|\Sset|\tilde{L}\log n}\right).
\end{align*}
\vspace{-0.2in}\end{theorem}

The bound in \cref{thm:ts_regret_prior} has two linear terms in $n$, with $\delta$ and the high-probability error $\varepsilon$. Because the posterior over models is updated online, $\varepsilon$ should decrease as more rounds are observed online, meaning our bound is overly conservative. Nevertheless, some offline model-learning methods, such as tensor decomposition \citep{tensor_decomposition}, yield $\varepsilon = \mathcal{O}(1/\sqrt{n})$ for an offline dataset of size $n$. Thus our bound is not vacuous. We can formally relate $\varepsilon$ and $\delta$ using the tails of the conditional reward distributions. Let $\mu(a, x, s; \theta) - \bar{\mu}(a, x, s)$ be $v^2$-sub-Gaussian for all $a$, $x$, and $s$, where the random quantity is $\theta \sim P_1$. Then for any $\delta > 0$, we have that $\varepsilon = \mathcal{O}(v\sqrt{\log(K|\Xset||\Sset|/\delta}))$ satisfies the conditions on $\varepsilon$ and $\delta$ needed for \cref{thm:ts_regret_prior}.

Among non-stationary contextual bandit algorithms, Exp4.S has near-optimal regret of $\tilde{O}(\sqrt{|\Sset|nL})$ for $|\Sset|$ experts, when $L$ is known, and $\tilde{O}(\sqrt{|\Sset|n}\, L)$, otherwise \citep{nonstationary_contextual_colt}.
Note that the tightness of our Bayes regret bound is limited by the sliding-window algorithm \mucb. Though conceptually simple and able to yield sublinear regret, \mucb likely yields a conservative Bayes regret bound. In addition, because our algorithms naturally leverage the stochasticity of the environment, we significantly outperform near-optimal algorithms, like Exp4.S, empirically. We demonstrate this in \cref{sec:experiments}.

\vspace{-0.05in}
\section{Experiments}
\label{sec:experiments}
\vspace{-0.05in}

In this section, we evaluate our algorithms on both synthetic and real-world datasets. We compare the following methods: (i) \textbf{CD-UCB}: UCB/LinUCB \citep{ucb,linucb} with a change-point detector as in \citet{cao_cpdetect}; (ii) \textbf{CD-TS}: TS/LinTS \citep{lints,lints_2} with the same change-point detector; (iii) \textbf{Exp.S}: \mexp/\mmexp using offline reward model as experts, where each expert takes the best action as measured by its conditional reward model \citep{exp3_s,nonstationary_contextual_colt}; (iv) \textbf{mTS, umTS}: our proposed TS algorithms \mts, \umts. 

In contrast to our method, the first two baselines do not use an offline model, but augment traditional bandit algorithms with a change-point detector that resets the algorithm when a change is detected. When there is no context, \citet{cao_cpdetect} proposed a detector with near-optimal guarantees and state-of-the-art empirical performance. The last baseline modifies adversarial algorithms Exp3/Exp4 by enforcing a lower-bound on the expert weights; this has near-optimal regret in piecewise-stationary bandits \citep{exp3_s}.

\begin{figure*}[t]
\centering
\begin{minipage}{0.33\textwidth}
 \includegraphics[width=0.9\textwidth]{figures/sim_fixed.pdf}
    \caption{Mean and standard error of regret across $100$ runs with fixed change-points.}
    \label{fig:sim_regret}
\end{minipage}
\begin{minipage}{0.66\textwidth}
\begin{minipage}{0.5\textwidth}
 \includegraphics[width=0.9\textwidth]{figures/sim_known.pdf}
 \subcaption{Random latent states, fixed model.}
\end{minipage}
\begin{minipage}{0.5\textwidth}
 \includegraphics[width=0.9\textwidth]{figures/sim_bayes_prior.pdf}
 \subcaption{Random latent states, uncertain model.}
\end{minipage}
\vspace{0.05in}\caption{Mean and standard error of reward across $100$ runs in synthetic scenarios.}
\label{fig:sim_bayes_reward}
\end{minipage}
\vspace{-0.1in}\end{figure*}

\vspace{-0.05in}
\subsection{Synthetic Experiments}
\vspace{-0.05in}

We artificially create a non-stationary multi-armed bandit without context, with $\Aset = [5]$ and $\Sset = [5]$. Mean rewards for are sampled uniformly at random $\mu(a, s) \sim \mathsf{Uniform}(0, 1)$ for each $a \in \Aset, s \in \Sset$. Rewards are drawn i.i.d. from $P(\cdot \mid a, s) = \mathcal{N}(\cdot \mid \mu(a, s), \sigma^2)$ with $\sigma = 0.5$. We use a horizon of $n = 2000$ as a primary application we are concerned with is fast personalization.

For \ucb, \ts, we use a change-point detector that computes the sum of the rewards for each arm in the past $\tau / 2$ rounds, and the $\tau / 2$ rounds before that. If the absolute value of their difference is greater than a threshold $b$, a change is detected. Following \citet{cao_cpdetect}, the window length parameter was tuned to $\tau=100$ to minimize regret, and the threshold is chosen to be $b = \sigma\sqrt{\tau \log(2|\Aset|n^2) / 2}$. 

First, we consider the specific setting where fixed changes between latent states occur exactly every $200$ rounds. We give model-based algorithms \mexp and \mts the true mean rewards. In \cref{fig:sim_regret}, we report the cumulative regret of all methods across $100$ runs. Because we use a short horizon, baseline bandit algorithms that are more sample inefficient and drastically outperformed by our method which leverages a prior model.
Next, we assume random latent state changes according to a transition matrix with $0.0025$ probability of uniformly changing to another latent state, meaning the latent state changes every $200$ rounds in expectation.
In \cref{fig:sim_bayes_reward}, we show the mean reward of each method across $100$ independent runs under two different scenarios: (i) the reward model and transition matrix are known; (ii) mean rewards are drawn from a Gaussian prior $\mathcal{N}(\mu(a, s), \sigma_0^2)$ with $\sigma_0 = 0.2$, and transitions are drawn from a Dirichlet prior with $0.005$ expected probability of transition, i.e. for a state $s$, $\phi_s \sim \mathsf{Dir}((\alpha_{s, s'})_{s' \in \Sset})$ where parameters $(\alpha_s)_{s \in \Sset}$ satisfy $\alpha_{s, s'} = 796$ if $s' = s$ and $1$ otherwise. 

When a prior is given, \mts uses the mean of the prior as if it were the true model parameter, resulting in a performance gap due to model misspecification. We report the average reward in \cref{fig:sim_bayes_reward} because the regret becomes dominated by model error in the given \say{offline} model. In all cases, the model-based Thompson sampling algorithms outperform all baselines by a significant margin, and do not require extensive hyperparameter tuning as the baselines did. Also, as uncertainty is introduced in the offline model, \mts performs worse than \umts, which accounts for model uncertainty.

\vspace{-0.05in}
\subsection{MovieLens Experiments}
\vspace{-0.05in}

We also assess the performance of our algorithms on the MovieLens 1M dataset \citep{movielens}, a popular collaborative filtering dataset, where $6040$ users rate $3883$ movies. Each movie has a set of genres. We filter the dataset to include only users who rated at least $200$ movies and movies rated by at least $200$ users. This results in $1353$ users and $1124$ movies.
We randomly select $50\%$ of all ratings as our \say{offline} training set, and use the remaining $50\%$ as a test set, giving sparse ratings matrices $M_{\text{train}}$ and $M_{\text{test}}$. We complete each matrix using least-squares matrix completion \citep{pmf} with rank $20$ to yield a low prediction error without overfitting. The learned training (test) factors are $M_{\text{train}} = \hat{U} \hat{V}^\top$ ($M_{\text{test}} = U V^\top$). In the training (test) set, user $i$ and movie $j$ correspond to the rows in the corresponding matrix, $\hat{U}_i$ ($U_i$) and $\hat{V}_i$ ($V_j$).

We define a non-stationary latent contextual bandit instance with $\Aset = [20]$ and $\Sset = [5]$ as follows. We use $k$-means clustering on the rows of $U$ to cluster users into $5$ clusters, where $5$ is the largest value that yields evenly-sized clusters. Motivated by prior work \citep{nonstationary_contextual_sigir}, we create a \say{superuser} by randomly sampling $5$ users $i_1, \dots, i_5$, one from each cluster; for latent state $s$, the superuser behaves according to the user $i_s$. Note that different superusers will have a different set of behavior modes, which is often true in practice. The transition matrix that governs the dynamics of the superuser is given by the linear combination
$
  P(s' \mid s; \phi_*) = 0.9 J(s, s') + 0.1 K(s, s'),
$
where $J(s, s') = 1 - p$ if $s' = s$ and $p / (|\Sset| - 1)$ otherwise, and $K(s, s') \propto \exp(||U_{i_{s'}} - U_{i_s}||_2^2)$. Here $J$ is used to ensure changes are infrequent, and $K$ to make transitions to similar latent states more likely. We let $p = 0.9975$ so that changes occur roughly every $400$ rounds with $n=2000$.

A run of a non-stationary contextual bandit proceeds as follows. A superuser $i_1, \hdots, i_5$ is sampled at random as described above. In each round, a latent state $S_t$ is generated according to $S_{t-1}$ and the transition matrix. Then, $20$ genres, then a movie for each genre, are both uniformly sampled from the set of all genres, movies, respectively, creating a set of diverse movies. Context $X_t \in \mathbb{R}^{20 \times 20}$ is a matrix where the rows are the training feature vectors of the sampled movies, that is movie $j$ has a vector $\hat{V}_j$. The agent chooses among movies in $X_t$. The reward for recommending movie $j$ to the superuser under state $S_t = s$ is drawn from $R_t \sim \mathcal{N}(U_{i_s}^\top V_j, 0.25)$, the product of the test user and movie vectors as its mean. Note that both $U$ and $V$ are unknown to the learning agent.

Our baselines \linucb, \lints are given movie vectors from the training set as context, and need to only learn the user vector. We could not find prior work that performed change detection in linear bandits, so we propose an adaption of the one by \citet{cao_cpdetect} to the linear case. Specifically, for round $t$ and window size $\tau$, the detector computes the least-squares solution $\hat{W}_{t}$ with features $X_{t, A_t}$ and rewards $R_t$ for the past $\tau / 2$ rounds of data, and $\hat{W}_{t}'$ for the $\tau / 2$ rounds before that. Let $\hat{\Sigma}_t = \sum_{t - \tau}^t X_{t, A_t}^\top X_{t, A_t}$ be the empirical covariance matrix. The detector fires when $||\hat{W}_{t} - \hat{W}_t'||_{\hat{\Sigma}_t} \geq b$; for matrix $M$ and weights $\Sigma$, the weighted norm is given by $||M||_{\Sigma} = \sqrt{M^T \Sigma M}$. Here both $\tau$ and $b$ are tuned $\tau = 100$ and $b = 13$ by maximizing reward during evaluation.

We learn a model \say{offline} in the same way as the true model is constructed, except using the training set. Our offline model consists of $5$ clusters of users derived from $k$-means clustering on users $\hat{U}$ in the training set. For each latent state, the prior given to our algorithms is a Gaussian prior with the corresponding cluster's mean and covariance.
Similarly, we estimate a transition matrix $\hat{\phi}$ using the same process as $\phi_*$ but using the cluster means on the training set instead of user features on the test set. We give a Dirichlet prior with parameters $(\alpha_{s, s'})_{s, s' \in \Sset \times \Sset}$ where $\alpha_{s, s'} = 800\,P(s' \mid s; \hat{\phi})$.

We evaluate on $100$ superusers, and show the mean reward in \cref{fig:movielens_reward}.
Again, the model-based algorithms outperform finely-tuned baselines by a significant margin, especially in the short horizon. Since the offline model is misspecified due to the train-test split, \umts improves upon \mts in the long term, as it refines its model parameters online. 

\begin{figure}
    \centering
    \includegraphics[width=0.3\textwidth]{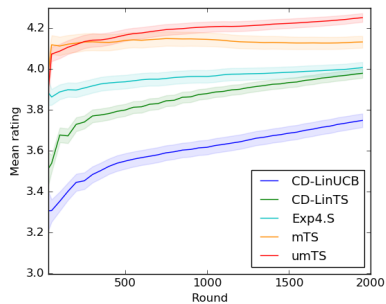}
    \caption{Mean and standard error of reward across 100 ``superusers" in MovieLens 1M.}
    \label{fig:movielens_reward}
\vspace{-0.1in}\end{figure}

\vspace{-0.05in}
\section{Related Work}
\vspace{-0.05in}
\label{sec:related work}

\paragraph{Non-stationary Bandits.} This topic has been studied extensively \citep{d_ucb,sw_ucb,exp3_s}. First works adapted to changes passively by weighting rewards, either by exponential discounting \citep{d_ucb} or by considering recent rewards in a sliding window~\citep{sw_ucb}. The latter yields a $\tilde{\mathcal{O}}(K\sqrt{nL}/\Delta^2)$ gap-dependent bound when $L$ is known. In the adversarial setting \citep{exp3_s,SHIFTBAND}, adaptation can be achieved by bounding the weights of experts from below. This leads to $\tilde{\mathcal{O}}(\sqrt{n|\Sset|L})$ gap-free switching regret, where $|\Sset|$ is the number of experts. \citet{Rexp3} periodically reset a base bandit algorithm and attain $\tilde{\mathcal{O}}(n^{2/3}V_T^{1/3})$ regret, where $V_T$ is the total variation under smooth changes. Other works monitor reward distributions and reset the bandit algorithm when a change is detected ~\citep{wmd,CUSUM}. \citet{cpbayesian} proposed augmenting Thompson sampling with a Bayesian change-point detector, but provide no regret guarantee. \citet{cao_cpdetect} proposed a simple near-optimal change-point detector that yields $\tilde{\mathcal{O}}(\sqrt{nKL})$ regret. In linear bandits, several recent paper studied passive adaptation of UCB algorithms \citep{nonstationary_linucb_aistats19,nonstationary_linucb_neurips19,nonstationary_linucb_aistats20}. This yields $\tilde{\mathcal{O}}(n^{2/3}P_T^{1/3})$ regret, where $P_T$ measures the total variation in an unknown weight vector. \citet{nonstationary_contextual_colt} provided several contextual algorithms with similar regret to ours, with the best algorithm matching the Exp4.S bound of $\tilde{\mathcal{O}}(\sqrt{n|\Sset|L})$. All above methods forget the past, discount it, or are adversarial. This is a major drawback when the environment changes in a structured manner.

\paragraph{Latent Bandits.} Our work is also related to latent bandits \citep{latent_bandits,latent_contextual_bandits}. Here the latent state is fixed across rounds and algorithms compete with standard bandit strategies, such as UCB ~\citep{ucb,linucb} or Thompson sampling \citep{lints,lints_2}. \citet{latent_bandits} derived UCB algorithms in the multi-armed case without context under the extremes when the mean conditional rewards are either known or need to be estimated completely online. \citet{latent_contextual_bandits} extended it to contextual bandits where policies are learned offline and selected online using Exp4. Bayesian policy reuse (BPR) \citep{bpr} selects offline-learned policies by maintaining a belief over the optimality of each policy, but no regret analysis exists. Recently, \citet{latent_bandits_revisited} proposed and analyzed TS algorithms with complex offline-learned models. Our work is the first to extend latent bandits to non-stationary environments by considering a latent state that evolves according to a transition model, which is known or sampled from a known prior.

\vspace{-0.05in}
\section{Conclusions}
\label{sec:conclusions}
\vspace{-0.05in}

We study non-stationary latent bandits, where the conditional rewards depend on an evolving discrete latent state. Given the plethora of rich offline models, we consider a setting where an offline-learned model can be used naturally by Thompson sampling to identify the latent state online. Prior algorithms for non-stationary bandits adapt by forgetting the past, discounting it, or are adversarial. We avoid this by leveraging the stochastic latent structure of our problem and thus can outperform prior works empirically by a large margin. Our approach is contextual, aware of uncertainty, and we analyze it by a reduction to a sliding-window UCB algorithm. Though our analysis is conservative, our work can be viewed as a stepping stone for analyzing the Bayes regret of Thompson sampling in more complex graphical models than a single fixed latent state \citep{latent_bandits,latent_contextual_bandits,latent_bandits_revisited}.

\bibliographystyle{plainnat}
\bibliography{references}

\begin{thebibliography}{35}
\providecommand{\natexlab}[1]{#1}
\providecommand{\url}[1]{\texttt{#1}}
\expandafter\ifx\csname urlstyle\endcsname\relax
  \providecommand{\doi}[1]{doi: #1}\else
  \providecommand{\doi}{doi: \begingroup \urlstyle{rm}\Url}\fi

\bibitem[Abbasi-yadkori et~al.(2011)Abbasi-yadkori, Pál, and
  Szepesvári]{linucb}
Yasin Abbasi-yadkori, Dávid Pál, and Csaba Szepesvári.
\newblock Improved algorithms for linear stochastic bandits.
\newblock In \emph{Neural Information Processing Systems}, 2011.

\bibitem[Abeille and Lazaric(2016)]{lints_2}
Marc Abeille and Alessandro Lazaric.
\newblock Linear thompson sampling revisited.
\newblock In \emph{Electronic Journal of Statistics}, 2016.

\bibitem[Agrawal and Goyal(2013)]{lints}
Shipra Agrawal and Navin Goyal.
\newblock Thompson sampling for contextual bandits with linear payoffs.
\newblock In \emph{International Conference on Machine Learning}, 2013.

\bibitem[Anandkumar et~al.(2014)Anandkumar, Ge, Hsu, Kakade, and
  Telgarsky]{tensor_decomposition}
Anima Anandkumar, Rong Ge, Daniel~J. Hsu, Sham~M. Kakade, and Matus Telgarsky.
\newblock Tensor decompositions for learning latent variable models.
\newblock In \emph{Journal of Machine Learning Research}, 2014.

\bibitem[Auer(2002)]{auer2002}
Peter Auer.
\newblock Finite-time analysis of the multiarmed bandit problem.
\newblock In \emph{Machine Learning}, 2002.

\bibitem[Auer(2003)]{SHIFTBAND}
Peter Auer.
\newblock Using confidence bounds for exploitation-exploration trade-offs.
\newblock In \emph{Journal of Machine Learning Research}, 2003.

\bibitem[Auer et~al.(2002{\natexlab{a}})Auer, Cesa-Bianchi, and Fischer]{ucb}
Peter Auer, Nicolò Cesa-Bianchi, and Paul Fischer.
\newblock Finite-time analysis of the multiarmed bandit problem.
\newblock \emph{Machine Learning}, 2002{\natexlab{a}}.

\bibitem[Auer et~al.(2002{\natexlab{b}})Auer, Cesa-Bianchi, Freund, and
  Schapire]{exp3_s}
Peter Auer, Nicolò Cesa-Bianchi, Yoav Freund, and Robert~E Schapire.
\newblock The nonstochastic multiarmed bandit problem.
\newblock In \emph{SIAM journal on computing}, 2002{\natexlab{b}}.

\bibitem[Besbes et~al.(2014)Besbes, Gur, and Zeevi]{Rexp3}
Omar Besbes, Yonatan Gur, and Assaf Zeevi.
\newblock Stochastic multi-armed-bandit problem with non-stationary rewards.
\newblock In \emph{Advances in Neural Information Processing Systems}, 2014.

\bibitem[Cao et~al.(2019)Cao, Wen, Kveton, and Xie]{cao_cpdetect}
Yang Cao, Zheng Wen, Branislav Kveton, and Yao Xie.
\newblock Nearly optimal adaptive procedure with change detection for
  piecewise-stationary bandit.
\newblock In \emph{International Conference on Artificial Intelligence and
  Statistics}, 2019.

\bibitem[Chapelle and Li(2012)]{chapelle11empirical}
Olivier Chapelle and Lihong Li.
\newblock An empirical evaluation of {Thompson} sampling.
\newblock In \emph{Neural Information Processing Systems}, pages 2249--2257,
  2012.

\bibitem[Cheung et~al.(2019)Cheung, Simchi-Levi, and
  Zhu]{nonstationary_linucb_aistats19}
Wang~Chi Cheung, David Simchi-Levi, and Ruihao Zhu.
\newblock Learning to optimize under non-stationarity.
\newblock In \emph{International Conference on Artificial Intelligence and
  Statistics}, 2019.

\bibitem[Clyde and George(2004)]{model_uncertainty}
Merlise Clyde and Edward~I. George.
\newblock Model uncertainty.
\newblock \emph{Statistical Science}, 2004.

\bibitem[Doucet et~al.(2013)Doucet, Gordon, and de~Freitas]{smc}
Arnaud Doucet, Neil Gordon, and Nando de~Freitas.
\newblock \emph{Sequential Monte Carlo Methods in Practice}.
\newblock Springer New York, 2013.

\bibitem[Garivier and Moulines(2008)]{sw_ucb}
Aurélien Garivier and Eric Moulines.
\newblock On upper-confidence bound policies for non-stationary bandit
  problems.
\newblock In \emph{International Conference on Algorithmic Learning Theory},
  2008.

\bibitem[Harper and Konstan(2015)]{movielens}
F.~Maxwell Harper and Joseph~A. Konstan.
\newblock The {MovieLens} datasets: History and context.
\newblock In \emph{ACM Transactions on Interactive Intelligent Systems (TiiS)},
  2015.

\bibitem[Hartland et~al.(2007)Hartland, Baskiotis, Gelly, Sebag, and
  Teytaud]{hartland2007}
Cédric Hartland, Nicolas Baskiotis, Sylvain Gelly, Michèle Sebag, and Olivier
  Teytaud.
\newblock Change point detection and meta-bandits for online learning in
  dynamic environments.
\newblock 2007.

\bibitem[Hong et~al.(2020)Hong, Kveton, Zaheer, Chow, Ahmed, and
  Boutilier]{latent_bandits_revisited}
Joey Hong, Branislav Kveton, Manzil Zaheer, Yinlam Chow, Amr Ahmed, and Craig
  Boutilier.
\newblock Latent bandits revisited.
\newblock \emph{CoRR}, abs/2006.08714, 2020.

\bibitem[Kocsis and Szepesvari(2006)]{d_ucb}
Levente Kocsis and Csaba Szepesvari.
\newblock Discounted ucb.
\newblock In \emph{2nd PASCAL Challenges Workshop}, 2006.

\bibitem[Lai and Robbins(1985)]{lai1985}
T.L Lai and Herbert Robbins.
\newblock Asymptotically efficient adaptive allocation rules.
\newblock In \emph{Advances in applied mathematics}, 1985.

\bibitem[Lattimore and Szepesvári(2019)]{bandit_book}
Tor Lattimore and Csaba Szepesvári.
\newblock \emph{Bandit Algorithms}.
\newblock Cambridge University Press, 2019.
\newblock \doi{10.1017/9781108571401}.

\bibitem[Liu et~al.(2018)Liu, Lee, and Shroff]{CUSUM}
Fang Liu, Joohyun Lee, and Ness~B. Shroff.
\newblock A change-detection based framework for piecewise-stationary
  multi-armed bandit problem.
\newblock In \emph{AAAI Conference on Artificial Intelligence}, 2018.

\bibitem[Luo et~al.(2018)Luo, Wei, Agarwal, and
  Langford]{nonstationary_contextual_colt}
Haipeng Luo, Chen-Yu Wei, Alekh Agarwal, and John Langford.
\newblock Efficient contextual bandits in non-stationary worlds.
\newblock In \emph{Conference on Learning Theory}, 2018.

\bibitem[Maillard and Mannor(2014)]{latent_bandits}
Odalric-Ambrym Maillard and Shie Mannor.
\newblock Latent bandits.
\newblock In \emph{International Conference on Machine Learning}, 2014.

\bibitem[Mellor and Shapiro(2013)]{cpbayesian}
Joseph Mellor and Jonathan Shapiro.
\newblock Thompson sampling in switching environments with bayesian online
  change detection.
\newblock In \emph{International Conference on Artificial Intelligence and
  Statistics}, 2013.

\bibitem[Rosman et~al.(2016)Rosman, Hawasly, and Ramamoorthy]{bpr}
Benjamin Rosman, Majd Hawasly, and Subramanian Ramamoorthy.
\newblock Bayesian policy reuse.
\newblock In \emph{Machine Learning}, 2016.

\bibitem[Russac et~al.(2019)Russac, Vernade, and
  Capp\'{e}]{nonstationary_linucb_neurips19}
Yoan Russac, Claire Vernade, and Olivier Capp\'{e}.
\newblock Weighted linear bandits for non-stationary environments.
\newblock In \emph{Neural Information Processing Systems}, 2019.

\bibitem[Russo and {Van Roy}(2013)]{russo_posterior_sampling}
Daniel Russo and Benjamin {Van Roy}.
\newblock Learning to optimize via posterior sampling.
\newblock \emph{CoRR}, abs/1301.2609, 2013.

\bibitem[Salakhutdinov and Mnih(2008)]{pmf}
Ruslan Salakhutdinov and Andriy Mnih.
\newblock Probabilistic matrix factorization.
\newblock \emph{Neural Information Processing Systems}, 2008.

\bibitem[S{\"a}rkk{\"a}(2013)]{sarkka2013bayesian}
Simo S{\"a}rkk{\"a}.
\newblock \emph{{Bayesian} Filtering and Smoothing}.
\newblock Cambridge University Press, 2013.

\bibitem[Thompson(1933)]{thompson33likelihood}
William~R. Thompson.
\newblock On the likelihood that one unknown probability exceeds another in
  view of the evidence of two samples.
\newblock \emph{Biometrika}, 25\penalty0 (3-4):\penalty0 285--294, 1933.

\bibitem[Wu et~al.(2018)Wu, Iyer, and Wang]{nonstationary_contextual_sigir}
Qingyun Wu, Naveen Iyer, and Hongning Wang.
\newblock Learning contextual bandits in a non-stationary environment.
\newblock In \emph{ACM SIGIR Conference on Research and Development in
  Information Retrieval}, 2018.

\bibitem[Yu and Mannor(2009)]{wmd}
Jia~Yuan Yu and Shie Mannor.
\newblock Piecewise-stationary bandit problems with side observations.
\newblock In \emph{International Conference on Machine Learning}, 2009.

\bibitem[Zhao et~al.(2020)Zhao, Zhang, Jiang, and
  Zhou]{nonstationary_linucb_aistats20}
Peng Zhao, Lijun Zhang, Yuan Jiang, and Zhi-Hua Zhou.
\newblock A simple approach for non-stationary linear bandits.
\newblock In \emph{International Conference on Artificial Intelligence and
  Statistics}, 2020.

\bibitem[Zhou and Brunskill(2016)]{latent_contextual_bandits}
Li~Zhou and Emma Brunskill.
\newblock Latent contextual bandits and their application to personalized
  recommendations for new users.
\newblock In \emph{International Joint Conferences on Artificial Intelligence},
  2016.

\end{thebibliography}

\clearpage
\onecolumn
\appendix

\section{Proofs}

Our proofs rely on the following concentration inequality, which is a straightforward extension of the Azuma-Hoeffding inequality to sub-Gaussian random variables. This was used and proved by \citet{latent_bandits_revisited}.

\begin{proposition}
\label{thm:azuma_general}
Let $(Y_t)_{t \in [n]}$ be a martingale difference sequence with respect to filtration $(\mathcal{F}_t)_{t \in [n]}$, that is $\E{}{Y_t \mid \mathcal{F}_{t - 1}} = 0$ for any $t \in [n]$. Let $Y_t \mid \mathcal{F}_{t - 1}$ be $\sigma^2$-sub-Gaussian for any $t \in [n]$. Then for any $\varepsilon > 0$, 
\begin{align*}
  \prob{\Big|\sum_{t = 1}^n Y_t\Big| \geq \varepsilon}
  \leq 2 \exp\left[- \frac{\varepsilon^2}{2 n \sigma^2}\right]\,.
\end{align*}
\end{proposition}

\subsection{Proof of \cref{lem:ucb_regret}}
\label{sec:ucb_regret_proof}

Recall that we have the following fixed quantities: true reward parameters $\theta_*$, latent state sequence $s_{1:n}$, and the number of stationary segments $L$. Note that we can decompose the regret as
\begin{align}
\begin{split}
  \Regret(n; \theta_*, s_{1:n})
  & = \E{}{\sum_{t=1}^n \left(\mu(A_{t, *}, X_t, s_t; \theta_*) - \mu(A_t, X_t, s_t; \theta_*)\right)} \\
  & =  \E{}{\sum_{t=1}^n \left(\mu(A_{t, *}, X_t, s_t; \theta_*) - U_t(A_t)\right)} 
  + \E{}{\sum_{t=1}^n \left(U_t(A_t) - \mu(A_t, X_t, S_t; \theta_*)\right)} \\
  & \leq \E{}{\sum_{t=1}^n \left(\mu(A_{t, *}, X_t, s_t; \theta_*) - U_t(A_{t, *})\right)} 
  + \E{}{\sum_{t=1}^n \left(U_t(A_t) - \mu(A_t, X_t, S_t; \theta_*)\right)}\,.
\end{split}
\label{eqn:ucb_regret_decomposition}
\end{align}
This is because for all rounds $t \in [n]$, we choose $A_t = \arg\max_{a \in \Aset} U_t(a)$, which means $U_t(A_t) \geq U_t(A_{t, *})$.

Let $\mathcal{T}$ be a set of all rounds $t$ that are not close to any change-point, that is $s_\ell = s_t$ for all $\ell \in \{t - \tau+1, \hdots, t\}$. Note that this includes all rounds where the last $\tau$ rounds have the same latent state as that round. 
Let
\begin{align}
    E_t = \left\{
    \forall s \in \Sset: \, 
    \abs{\sum_{\ell = \min\{1, t-\tau\}}^{t-1} 
    \indicator{B_\ell = s} \left(\mu(A_\ell, X_\ell, s_\ell; \theta_*) - R_\ell\right)} \leq \sigma \sqrt{6 N_t(s) \log n} \right\}
    \label{eqn:sw_ucb_event}
\end{align} 
be the event that the total realized reward under each played latent state is close to its expectation. Let $E = \cap_{t \in \mathcal{T}} E_t$ be the event that this holds for all rounds not close to a change-point, and $\bar{E}$ be its complement. Then we can bound the expected $n$-round regret as
\begin{align}
\begin{split}
\label{eqn:sw_regret_event_decomposition}
    &\Regret(n; \theta_*, s_{1:n}) \\
    &\,\leq  L\tau + \E{}{\sum_{t \in \mathcal{T}} \left(\mu(A_{t, *}, X_t, s_t; \theta_*) - \mu(A_t, X_t, s_t; \theta_*)\right)}  \\
    &\,=  L\tau + \E{}{\indicator{\bar{E}}\sum_{t \in \mathcal{T}} \left(\mu(A_{t, *}, X_t, s_t; \theta_*) - \mu(A_t, X_t, s_t; \theta_*)\right)}
    +
    \E{}{\indicator{E}\sum_{t \in \mathcal{T}} \left(\mu(A_{t, *}, X_t, s_t; \theta_*) - \mu(A_t, X_t, s_t; \theta_*)\right)}\\
    &\,\leq  L\tau + \E{}{\indicator{\bar{E}} \sum_{t \in \mathcal{T}} \left(\mu(A_{t, *}, X_t, s_t; \theta_*) - \mu(A_t, X_t, s_t; \theta_*)\right)}  \\
    &\,\qquad + 
    \E{}{\indicator{E}\sum_{t \in \mathcal{T}} \left(\mu(A_{t, *}, X_t, s_t; \theta_*) - U_t(A_{t, *})\right)} + \E{}{\indicator{E}\sum_{t \in \mathcal{T}} \left(U_t(A_t) - \mu(A_t, X_t, s_t; \theta_*)\right)}\,,
\end{split}
\end{align}
where for the first inequality we upper bound the regret in rounds close to change-points by $L \tau$, and in the second we use the regret decomposition in \eqref{eqn:ucb_regret_decomposition}. We ignore the rounds within $\tau$ rounds of change-points because the empirical mean reward estimates over the those rounds are biased.

We first show that the probability of $\bar{E}$ occurring is low. Without context, this would follow immediately from Hoeffding's inequality. Since we have context generated by some random process, we instead turn to martingales. 

\begin{proposition}
\label{thm:ucb_concentration} Let $E_t$ be defined as in \eqref{eqn:sw_ucb_event} for all rounds $t$, $E = \cap_{t \in \mathcal{T}} E_t$, and $\bar{E}$ be its complement. Then $\prob{\bar{E}} \leq 2|\Sset|n^{-1}$.
\end{proposition}

\begin{proof}
Because the UCBs depend on which latent states are eliminated, the UCBs depend on the history, and the conditional action given observed context also depends on the history. 
For each latent state $s$ and round $t$, let $\mathcal{T}_{t, s}$ be the rounds where state $s$ was chosen among the past $\tau$ rounds. For round $\ell \in \mathcal{T}_{t, s}$, let $Y_\ell(s) = \mu(A_\ell, X_\ell, s_\ell;\theta_*) - R_\ell$. Observe that $Y_\ell(s) \mid X_\ell, H_\ell$ is $\sigma^2$-sub-Gaussian. This implies that $(Y_\ell(s))_{\ell \in \mathcal{T}_{t, s}}$ is a martingale difference sequence with respect to context and history $(X_\ell, H_\ell)_{\ell \in \mathcal{T}_{t, s}}$, or $\E{}{Y_\ell(s) \mid X_\ell, H_\ell} = 0$ for all rounds $\ell \in \mathcal{T}_{t, s}$.
% \todob{This is applied incorrectly on at least two levels. First, $t$ should go only over the rounds where $s$ is chosen. Second, $t$ should be limited to the past $\tau$ window. Let $\mathcal{T}_{t, s}$ be the past $\tau$ rounds where $s$ is chosen. Then our MDS is $(Y_t(s), (X_t, H_t))_{t \in \mathcal{T}_{t, s}}$. This sequence has a random length, of at most $\tau$. Hence the union bound over $u$ below. Note that the algebra below is correct. But we do not explain it correctly.}

For any round $t$, and state $s \in \Sset$, we have that $\mathcal{T}_{t, s}$ is a random quantity. First, we fix $|\mathcal{T}_{t, s}| = N_t(s) = u$ where $u \leq \tau$ and yield the following due to \cref{thm:azuma_general}, 
\begin{align*}
  \prob{\abs{\sum_{\ell \in \mathcal{T}_{t, s}} Y_\ell(s)} \geq  \sigma \sqrt{6 u \log n}}
  \leq 2\exp\left[-3\log n\right]
  = 2 n^{-3}\,.
\end{align*}
So, by the union bound, we have
\begin{align*}
  \prob{\bar{E}}
  \leq \sum_{t \in \mathcal{T}} \sum_{s \in \Sset} \sum_{u = 1}^{\tau}
  \prob{\abs{\sum_{\ell \in \mathcal{T}_{t, s}} Y_\ell(s)} \geq  \sigma \sqrt{6 u \log n}}
  \leq 2 |\Sset| n^{-1}\,.
\end{align*}
This concludes the proof.
\end{proof}

We can show that the second term in \eqref{eqn:sw_regret_event_decomposition} is small because the probability of $\bar{E}$ is small. 
Specifically, from \cref{thm:ucb_concentration}, and that total regret is bounded by $n$, we have that the second term in \eqref{eqn:sw_regret_event_decomposition} is bounded by $n\prob{\bar{E}} \leq 2|\Sset|$.

Next, we bound the third term in \eqref{eqn:sw_regret_event_decomposition}. For round $t \in \mathcal{T}$, the event $\mu(A_{t, *}, X_t, s_t; \theta_*) > U_t(A_{t, *})$ occurs only if $s_t \notin C_t$ also occurs. By the design of $C_t$ in \mucb, this happens only if $G_t(s_t) > \sigma\sqrt{6 N_t(s)\log n}$. Event $E_t$ says that the opposite is true for all states, including true state $s_t$. So the third term in \eqref{eqn:sw_regret_event_decomposition} is at most $0$. 

Now we consider the last term in \eqref{eqn:sw_regret_event_decomposition}. We know that $\mathcal{T}$ is composed of $L$ stationary segments. We bound the last term for each segment individually as follows.

\begin{proposition}
\label{thm:ucb_interval_width} Let $\mathcal{I} \subseteq \mathcal{T}$ be a stationary segment containing $m$ rounds. Then
\begin{align*}
    \E{}{\indicator{E} \sum_{t \in \mathcal{I}} \left(U_t(A_t) - \mu(A_t, X_t, s_t; \theta_*)\right)}
    \leq |\Sset|\lceil m /\tau \rceil + 2\sigma\sqrt{6|\Sset|\lceil m / \tau \rceil m\log{n}}.
\end{align*}
\vspace{-0.1in}\end{proposition}

\begin{proof}
To ease exposition, let the $m$ rounds in $\mathcal{I}$ be denoted $1, \hdots, m$. We can further divide $\mathcal{I}$ into intervals of length $\tau$ and the last with length of at most $\tau$. Let $1 = t_0 \leq t_1 \leq \hdots  \leq t_{\lceil m / \tau \rceil} = m$ partition $\mathcal{I}$ into such intervals. We can write,
\begin{align*}
    &\E{}{\indicator{E} \sum_{t \in \mathcal{I}} \left(U_t(A_t) - \mu(A_t, X_t, s_t; \theta_*)\right)} \\
    &\quad=
    \E{}{\indicator{E} \sum_{i = 1}^{\lceil m / \tau \rceil} \sum_{\ell = t_{i-1}}^{t_i} \left(\mu(A_\ell, X_\ell, B_\ell; \theta_*) - R_\ell\right)}
    + \E{}{\indicator{E} \sum_{i = 1}^{\lceil m / \tau \rceil} \sum_{\ell = t_{i-1}}^{t_i} \left(R_\ell - \mu(A_\ell, X_\ell, s_\ell; \theta_*)\right)} \\
    &\quad\leq 
    \E{}{\sum_{i = 1}^{\lceil m / \tau \rceil}\sum_{s \in S} (G_{t_i}(s) + 1)} 
    + \sum_{i = 1}^{\lceil m / \tau \rceil}\sum_{s \in S}  \sigma\sqrt{6N_{t_i}(s)\log n} \\
    &\quad\leq |\Sset|\lceil m / \tau \rceil + \sum_{s \in S} \sum_{i = 1}^{\lceil m / \tau \rceil} 2\sigma \sqrt{6N_{t_i}(s) \log n}.
\end{align*}
For each window $i$ of length $\tau$ and latent state $s$, we use that until the last round before $t_i$ where $s$ is selected, we have an upper bound on the total prediction error, given by the upper bound on the gap $G_{t_i}(s) \leq \sigma\sqrt{6N_{t_i}(s)\log n}$, where $G_{t_i}(s)$ is defined as in \eqref{eqn:ucb_gap} Recall that $E_{t_i}$, as defined in \eqref{eqn:sw_ucb_event}, occurring implies that the deviation of the realized reward from the true means bounded by $\sigma\sqrt{6N_{t_i}(s)\log n}$. Accounting for the last round where $s$ was chosen in window $i$ yields the right-hand side of the inequality. Applying the Cauchy-Schwarz inequality yields,
\begin{align*}
    \E{}{\indicator{E} \sum_{t \in \mathcal{I}} U_t(A_t) - \mu(A_t, X_t, s_t; \theta_*)}
    \leq |\Sset|\lceil m /\tau \rceil + 2\sigma\sqrt{6|\Sset|\lceil m / \tau \rceil m\log{n}},
\end{align*}
which is the desired upper bound.
\end{proof}

Now we can bound the last term in \eqref{eqn:sw_regret_event_decomposition} by combining \cref{thm:ucb_interval_width} across all $L$ stationary segments. Let $(\mathcal{I}_i)_{i \in [L]}$ denote the stationary segments, and segment $\mathcal{I}_i$ have length $m_i$. We have,
\begin{align*}
    \E{}{\indicator{E} \sum_{t \in \mathcal{T}} \left(U_t(A_t) - \mu(A_t, X_t, s_t; \theta_*)\right)}
    &= \E{}{\indicator{E} \sum_{i=1}^L\sum_{t \in \mathcal{I}_i} \left(U_t(A_t) - \mu(A_t, X_t, s_t; \theta_*)\right)} \\
    &\leq \sum_{i=1}^L |\Sset|\lceil m_i /\tau \rceil + 2\sigma\sqrt{6|\Sset|\lceil m_i / \tau \rceil m_i\log{n}} \\
    &\leq |\Sset|(n /\tau) + 2\sigma\sqrt{6|\Sset|(n/\tau) n\log{n}}.
\end{align*}
Here we use that for any segment $i$, we have $\lceil m_i /\tau \rceil \leq (m_i + \tau) / \tau$ for any number of rounds $m$, and that $\sum_i m_i = n - L\tau$ because we omitted rounds to close to a change-point.
Combining the bounds for all terms in \eqref{eqn:sw_regret_event_decomposition} yields,
\begin{align*}
    \Regret(n; \theta_*, s_{1:n})
    &\leq L\tau + 2|\Sset| + |\Sset|(n/\tau) + 2\sigma\sqrt{6|\Sset|(n/\tau)n \log n},
\end{align*}
When $L$ is known, we can solve for the optimal window length $\tau = \mathcal{O}(n^{2/3}\sqrt{|\Sset|\log n / L})$, which when substituted into the regret bound yields
$
    \Regret(n; \theta_*, s_{1:n}) = \mathcal{O}(n^{2/3}\sqrt{|\Sset|L\log n})
$
, as desired.

\subsection{Proof of \cref{thm:ts_regret}}
\label{sec:ts_regret_proof}

From the Bayes regret formulation in \eqref{eq:bayes_regret}, the true latent state sequence $S_{1:n} \in \Sset^n$ is random for a fixed transition model $\phi_*$. Here we still assume a fixed reward model $\theta_*$. We have that the optimal action $A_{t, *} = \arg\max_{a \in \Aset} \mu(a, X_t, S_t; \theta_*)$ is random not only due to context, but also latent state $S_t$. We also have that $L = \sum_{t = 1}^n \indicator{S_t \neq S_{t-1}}$ is random due to latent state sequence $S_{1:n}$. 

% \todob{The proof below is more complex than it should be. It should be simply an expectation of the previous upper bound on $ \Regret(n; \theta_*, S_{1 : n})$, which holds for any $\theta_*, S_{1 : n}$. In the previous proof, precisely write the upper bound on $\Regret(n; \theta_*, s_{1:n})$ without the big O notation and then take its expectation here. Think of the next proof in the same way.}
Similar to \citet{russo_posterior_sampling}, we reduce our analysis of \mts to analysis of \mucb as done in \cref{lem:ucb_regret}. We define $U_t(a) = \arg\max_{s \in C_t}\mu(a, X_t, s; \theta_*)$ where the $C_t$ is as in \mucb. Recall that the Bayes regret is given by \eqref{eq:bayes_regret}, and can be decomposed as \eqref{eqn:posterior_regret_decomposition}. In \cref{sec:ucb_regret_proof}, we bounded an equivalent regret decomposition for any $\theta_*, S_{1:n}$ and therefore also in expectation over $S_{1:n} \sim \phi_*$. We have the Bayes regret bound,
\begin{align*}
    \Bregret(n; \theta_*, \phi_*) 
    &= \E{}{\sum_{t=1}^n \left(\mu(A_{t, *}, X_t, S_t; \theta_*) - U_t(A_{t, *})\right) \mid \theta_*, \phi_*}
    + \E{}{\sum_{t=1}^n \left(U_t(A_t) - \mu(A_t, X_t, S_t; \theta_*)\right) \mid \theta_*, \phi_*} \\
    &\leq \E{S_{1:n} \sim \phi_*}{L\tau + 2|\Sset| + |\Sset|(n/\tau) + 2\sigma\sqrt{6|\Sset|(n/\tau)n \log n}},
\end{align*}
where we directly substitute the upper bound in \cref{lem:ucb_regret} inside the expectation. 

Since $\phi_*$ is known, we can define $p = 1 - \min_{s \in \Sset} P(s \mid s; \phi_*)$ as the maximum probability of a change occurring. Then number of change-points $L - 1$ is a binomial random variable, so that $\E{S_{1:n} \sim \phi_*}{L} = 1 + p n = \tilde{L}$. 
For optimal choice of $\tau = \mathcal{O}(n^{2/3}\sqrt{|\Sset|\log n / L})$, we can simplify the expectation over random $L$ to yield,
\begin{align*}
    \Bregret(n; \theta_*, \phi_*) 
    = \mathcal{O}\left(n^{2/3}\sqrt{|\Sset|\E{S_{1:n} \sim \phi_*}{L}\log n}\right) 
    = \mathcal{O}\left(n^{2/3}\sqrt{|\Sset|\tilde{L}\log n}\right),
\end{align*}
where we use Jensen's inequality and that the expression inside the expectation is concave in $L$.

\subsection{Proof of \cref{thm:ts_regret_prior}}
\label{sec:regret_proof_prior}

From the Bayes regret formulation in \eqref{eq:bayes_regret_2}, both the reward and transition model parameters $\theta_*, \phi_*$ are now random according to priors $P_1(\theta), P_1(\phi)$, respectively. We have that the optimal action $A_{t, *} = \arg\max_{a \in \Aset} \mu(a, X_t, S_t; \theta_*)$ is random due to context $X_t$, latent state $S_t$, and model $\theta_*$.

Recall that given prior $P_1(\theta)$, we have that $\bar{\mu}(a, x, s) = \int_\theta \mu(a, x, s; \theta) P_1(\theta) d \theta$ is the mean conditional reward marginalized with respect to the prior. We make one small change to \eqref{eqn:ucb_gap} in \mucb, which accounts for uncertainty: for round $t$ and state $s$, instead of acting according to true means $\mu(A_t, X_t, s; \theta_*)$, we act conservatively according to the mean marginalized over the prior $\bar{\mu}(A_t, X_t, s)$. Formally, the \say{gap} in \umucb is redefined as
\begin{align}
    \label{eqn:ucb_gap_prior}
    G_{t}(s) = \sum_{\ell = \max\{1, t - \tau\}}^{t-1} \indicator{B_\ell = s}\left(\bar{\mu}(A_\ell, X_\ell, s) - \varepsilon - R_\ell\right).
\end{align}
The additional $\varepsilon$ ensures that we do not mistakenly eliminate the true latent state from $C_t$ due to a prediction error. 

Our proof uses the following regret bound for \umucb, which is for a fixed $\theta_*$ sampled from the prior.
\begin{lemma}
For fixed model parameters $\theta_*$, assume that there exist $\varepsilon > 0$ such that $\theta_*$ satisfies the following:
$
\left\{\forall a \in \Aset, x \in \Xset, s \in \Sset: |\bar{\mu}(a, x, s) - \mu(a, x, s; \theta_*)| \leq \varepsilon\right\}.
$
Then for optimal $\tau$, the $n$-round regret of $\umucb$ is
\begin{align*}
    \Regret(n; \theta_*, s_{1:n})
    &= \mathcal{O}\left(\varepsilon n + n^{2/3}\sqrt{|\Sset|L\log n}\right).
\end{align*}
\label{lem:ucb_regret_2}
\vspace{-0.1in}\end{lemma}
\begin{proof}
% \todob{Too much overlap with Lemma 1. This increases complexity and the probability of making a mistake. Simply say that the proof follows that of Lemma 1. There are two difference. First, the extra $\varepsilon$ is introduced in the upper bound on $G_t(s)$ in Proposition 4. Say how it arises. Second, $|\bar{\mu}(a, x, s) - \mu(a, x, s; \theta_*)| > \varepsilon$ happens with probability of at most $\delta$ for at least one $a$, $x$, and $s$. In this case, we bound the regret trivially, and we do this as a first step in the proof of Lemma 1.}
We have the same regret decomposition for $n$-round regret, stated in \eqref{eqn:sw_regret_event_decomposition}.
The analysis proceeds similarly to \cref{sec:ucb_regret_proof}, only we need to additionally account for prediction error in the conditional mean rewards. We only highlight the differences, and defer other details of the proof to \cref{sec:ucb_regret_proof}.

Using \cref{thm:ucb_concentration}, and that the total regret is bounded by $n$, we again have the second term in \eqref{eqn:sw_regret_event_decomposition} can be bounded by,
$n\prob{\bar{E}} \leq 2|\Sset|$.
Bounding the third term in \eqref{eqn:sw_regret_event_decomposition} requires a slight change. For round $t \in \mathcal{T}$, we have that the event $\mu(A_{t, *}, X_t, s_t; \theta_*)  - U_t(A_{t, *}) > \varepsilon$ occurs only if $s_t \notin C_t$. By the design of $C_t$ in \umucb, this happens only if $G_t(s_t) > \sigma\sqrt{6 N_t(s)\log n}$, since
\begin{align*}
G_t(s_t) = \smashoperator[r]{\sum_{\ell = \max\{1, t - \tau\}}^{t-1}} 
    \indicator{B_\ell = s_t} \left(\bar{\mu}(A_\ell, X_\ell, s_t) - \varepsilon - R_\ell\right) 
    \leq \smashoperator[r]{\sum_{\ell = \max\{1, t - \tau\}}^{t-1}} 
    \indicator{B_\ell = s_*} \left(\mu(A_\ell, X_\ell, s_t; \theta_*) - R_\ell\right).
\end{align*}
Event $E_t$ says that the opposite is true for all states, including true state $s_t$. So the third term in \eqref{eqn:sw_regret_event_decomposition} is at most $\varepsilon n$. 

For the last term in \eqref{eqn:sw_regret_event_decomposition}, we need to account for the fact that $\varepsilon$ is included in the gap $G_t(s)$ for every round $t$ and state $s$.
To do so, we introduce a $\varepsilon n$ term in the expression as,
\begin{align*}
    \E{}{\indicator{E} \sum_{t \in \mathcal{T}} \left(U_t(A_t) - \mu(A_t, X_t, s_t; \theta_*\right)}
    \leq
    \varepsilon n + \E{}{\indicator{E} \sum_{t \in \mathcal{T}} \left(U_t(A_t) - \varepsilon + \mu(A_t, X_t, s_t; \theta_*)\right)}
\end{align*}
The second term on the right-hand side can be bounded the same way as in \cref{sec:ucb_regret_proof} by introducing the realized reward, and bounding the sum of confidence widths using the gap given in \eqref{eqn:ucb_gap_prior}. 

This yields the bound on total regret,
\begin{align*}
    \Regret(n; \theta_*, s_{1:n})
    &\leq L\tau + 2\varepsilon n + |\Sset|(n/\tau + 2) + 2\sigma\sqrt{6|\Sset|(n/\tau)n \log n}.
\end{align*} 
Solving for optimal window length in terms of $L$ yields $\tau = \mathcal{O}(n^{2/3}\sqrt{|\Sset|\log n / L})$, which when substituted into the regret gives
$
    \Regret(n; \theta_*, s_{1:n}) = \mathcal{O}(\varepsilon n + n^{2/3}\sqrt{|\Sset|L\log n})
$
, as desired.
\end{proof}

In order to prove \cref{thm:ts_regret_prior}, we again reduce to the proof of \cref{lem:ucb_regret_2} for \umucb. We define $U_t(a) = \arg\max_{s \in C_t}\mu(a, X_t, s; \theta_*)$ as in \umucb. We also define event 
$$
\mathcal{E} = \{\forall a \in \Aset, x \in \Xset, s \in \Sset: \abs{\bar{\mu}(a, x, s) - \mu(a, x, s; \theta_*)} \leq \varepsilon\},
$$
for when the sampled true model $\theta_*$ behaves close to expected, and $\bar{\mathcal{E}}$ as its complement. 
If $\mathcal{E}$ does not hold, then the best possible upper bound on regret is $n$; fortunately, we assume in the statement of the theorem that the probability of that occurring is bounded by $\delta$. So we can bound the $n$-round Bayes regret as
\begin{align*}
    \Bregret(n)
    = \E{}{\indicator{\bar{\mathcal{E}}} \Regret(n, \theta_*, s_{1:n})} + \E{}{\indicator{\mathcal{E}} \Regret(n; \theta_*, s_{1:n})}
    \leq \delta n + \E{}{\indicator{\mathcal{E}} \Regret(n; \theta_*, s_{1:n})}.
\end{align*}
The second term can be decomposed as in \eqref{eqn:posterior_regret_decomposition} and bounded by \cref{lem:ucb_regret_2} as the bound in the lemma is worst-case over any model parameters $\theta_*$ and sequence $S_{1:n}$. We have the Bayes regret bound,
\begin{align*}
    \Bregret(n)
    &\leq \delta n + 2\varepsilon n + \E{\phi_* \sim P_1}{L\tau + 2|\Sset| + |\Sset|(n/\tau) + 2\sigma\sqrt{6|\Sset|(n/\tau)n \log n}}.
\end{align*}
Here $\phi_*$ is random, and hence number of stationary segments $L$ is also random. Let $p$ denote the maximum probability of change, i.e., for fixed $\phi_*$, we have $p = 1 - \min_{s \in \Sset} P(s \mid s; \phi_*)$ as in \cref{thm:ts_regret}. Unlike in \cref{thm:ts_regret}, we have that $p$ is random as well due to randomness in $\phi_*$. Recall from the statement of \cref{thm:ts_regret_prior} that $(\alpha_{s, s'})_{s, s' \in \Sset \times \Sset}$ are the prior parameters of $P_1(\phi)$. We can write 
$
  \E{\phi_* \sim P_1}{p} \leq 1 - \min_{s \in \Sset} \alpha_{s, s} / \sum_{s' \in \Sset} \alpha_{s, s'} = \tilde{p}.
$
This means we can bound the expected value of $L$ as,
\begin{align*}
    \E{\phi_* \sim P_1}{L}
    = \E{\phi_* \sim P_1}{\E{S_{1:n} \sim \phi_*}{L \mid \phi_*}}
    \leq \E{\phi_* \sim P_1}{1 + pn}
    \leq 1 + \tilde{p}n 
    = \tilde{L}.
\end{align*}
Since the Bayes regret is still concave in $L$, we can apply the same trick as in \cref{sec:ts_regret_proof} using Jensen's inequality, and yield, for optimal choice of $\tau = \mathcal{O}(n^{2/3}\sqrt{|\Sset|\log n / L})$, the desired Bayes regret bound
\begin{align*}
    \Bregret(n) 
    = \mathcal{O}\left(\delta n + \varepsilon n  + n^{2/3}\sqrt{|\Sset|\E{\phi_* \sim P_1}{L}\log n}\right) 
    = \mathcal{O}\left(\delta n + \varepsilon n + n^{2/3}\sqrt{|\Sset|\tilde{L}\log n}\right).
\end{align*}

\end{document}